\documentclass[11pt]{article}

\usepackage{subcaption}
\usepackage{tikz}
\usepackage[margin=1in]{geometry}
\usepackage{hyperref}
\usepackage[numbers,sort&compress]{natbib}
\usepackage{amsmath, amssymb, amsthm}
\usepackage{booktabs} 
\hypersetup{
  colorlinks=true,
  linkcolor=blue,
  citecolor=blue,
  urlcolor=blue
}

\theoremstyle{plain}
\newtheorem{theorem}{Theorem}
\newtheorem{lemma}{Lemma}

\theoremstyle{definition}

\theoremstyle{remark}

\title{Fluid-Agent Reinforcement Learning}
\date{} 

\author{
Shishir Sharma$^{1,2}$ \quad
Doina Precup$^{1,2}$ \quad
Theodore J. Perkins$^{3}$ \\[4pt]
$^{1}$Mila – Quebec Artificial Intelligence Institute, Montreal, Canada \\
$^{2}$McGill University, Montreal, Canada \\
$^{3}$Ottawa Hospital Research Institute, University of Ottawa, Ottawa, Canada \\
\texttt{shishir.sharma@mail.mcgill.ca} \quad \texttt{doina.precup@mcgill.ca} \quad
\texttt{tperkins@ohri.ca}
}

\begin{document}
\maketitle


\begin{abstract}
The primary focus of multi‑agent reinforcement learning (MARL) has been to study interactions among a fixed number of agents embedded in an environment. However, in the real world, the number of agents is neither fixed nor known a priori. Moreover, an agent can decide to create other agents (for example, a cell may divide, or a company may spin off a division). In this paper, we propose a framework that allows agents to create other agents; we call this a fluid-agent environment. We present game-theoretic solution concepts for fluid-agent games and empirically evaluate the performance of several MARL algorithms within this framework. Our experiments include fluid variants of established benchmarks such as Predator–Prey and Level-Based Foraging, where agents can dynamically spawn, as well as a new environment we introduce that highlights how fluidity can unlock novel solution strategies beyond those observed in fixed-population settings. We demonstrate that this framework yields agent teams that adjust their size dynamically to match environmental demands.
\end{abstract}





         
\newcommand{\BibTeX}{\rm B\kern-.05em{\sc i\kern-.025em b}\kern-.08em\TeX}


\section{Introduction}

Historically, reinforcement learning (RL) \citep{Sutton}, has primarily focused on the viewpoint of a single agent embedded in some environment, acting to maximize its long-term return. Much has been achieved using this viewpoint, giving rise to algorithms that have varied applications in robotics \citep{robotics}, healthcare \citep{healthcare}, and economics \citep{economics}, among other fields. Multi-agent reinforcement learning (MARL) scenarios have also been explored, especially in the context of game-playing \citep{busoniu2008comprehensive}. World-class or even world-best players have been developed for two- or multi-player games such as Chess \citep{DBLP:journals/corr/abs-1712-01815}, GO \citep{Silver_2016} and Dota \citep{Berner2019Dota2W}. The vast majority of MARL research has focused on scenarios where there is a fixed number of agents, cooperating and/or competing with each other.

 However, in real-world circumstances, the number of agents with which one interacts can vary widely over time. Thus, while assuming that the number of agents is fixed over time leads to mathematical simplicity and allows a clean study of equilibria, sample complexity and other problems, it is also limiting. Moreover, in many real-world situations, agents take actions that {\em actively modify} the number of agents in the environment, for example by creating other agents or destroying them.  For instance, living beings procreate, thus creating new agents.  A business might acquire another business, or spin-off a subsidiary, or even decide to close down. In short, many real-world systems include a fluid collection of decision-making agents that may be in some mixture of cooperation and competition with each other, and that collection is constantly changing--in part due to the choices of the agents themselves, which may create or eliminate other agents.

In this paper, we formalize this setting by defining the notion of a \textit{fluid-agent} environment. We study a  simplified version of this setting, in which agents  are only able to create (spawn) other agents. 
We propose the Partially Observable Fluid Stochastic Games (POFSG) framework for modeling such fluid-agent games and prove the existence of Nash equilibria (NE) as well as subgame-perfect Nash equilibria (SPNE) in such games. We adapt existing popular environments, Predator-Prey  \citep{magym} and Level-Based Foraging \citep{papoudakis2021benchmarking}, to allow agent spawning and also introduce \emph{PuddleBridge}, a goal-based environment centered around dynamically leveraging fluidity to achieve optimal performance. Finally, we conduct an experimental analysis to examine several facets of spawn-only fluidity, including:
(a) the interplay between algorithmic inductive biases—joint versus per-agent optimization—and reward structure under fluidity;
(b) the ability of agents to adapt team size in response to task variability;
(c) the capacity to optimize the characteristics of spawned agents;
(d) the access to a richer policy space due to fluidity, allowing agents to switch between fluid and non-fluid strategies.

\section{Related Work}

To the best of our knowledge, no prior work has provided a formal treatment of fluid-agent environments as defined in this study. The closest practical analogue is AlphaStar \citep{vinyals2019grandmaster}, which permits unit creation and destruction, though without an explicit formalization. However, the AlphaStar game playing agent employs a single centralized policy to control all units. Thus, creating new units does not introduce additional decision-making entities. In contrast, our approach features decentralized, autonomous agents: multiple teammates act concurrently based on local observations (under CTDE training), and spawning introduces new decision-makers rather than merely expanding the action load of a single controller. 

Our framework is also conceptually related to \textit{ad-hoc teamwork} (AHT) \citep{Adhoc_Barett, Adhoc, 10.1007/978-3-031-20614-6_16}, which tackles the problem of agents coordinating with previously unseen teammates. However, there are several key differences between AHT and the setting we consider. \citet{10.1007/978-3-031-20614-6_16} characterize AHT by three assumptions: (i) no prior coordination such as through a joint training phase, (ii) no control of agents over teammates, and (iii) a collaborative objective. Our setting is not consistent with (i) and (ii). First, we permit centralized training with decentralized execution (CTDE), which means that teammates can coordinate during training. Second, agents directly determine team composition: they decide whether and when to spawn additional teammates and optimize for their capabilities. There is somewhat greater conceptual overlap between our work and open ad-hoc teamwork (OAHT) setting \citep{wang2024openadhocteamwork}, which permits team composition to evolve over time; however, newly introduced teammates are still externally provided rather than generated through strategic decisions within the system itself.

Our notion of population fluidity also relates to ideas from evolutionary game theory (EGT), where the size and composition of a population change over time as successful strategies replicate and weaker ones vanish. Classical models such as replicator dynamics \citep{taylor1978evolutionary, weibull1995evolutionary} describe this process at the population level, but agents in those frameworks do not make explicit decisions to reproduce or die—the dynamics arise automatically from payoff differences. In contrast, our formulation endogenizes this process: spawning is itself a strategic action taken by autonomous learners within the game. Thus, while EGT explains how advantageous behaviors can spread over time, our work examines how agents can intentionally control population change to improve collective performance.


\section{Background}

A MARL problem has two parts: (a) a game model defining how the agents interact and (b) a solution concept specifying what is desired or what can be achieved by the joint behavior of the agents \citep{marl-book}.  \textit{Partially Observable Stochastic Games (POSGs)} provide a natural framework for modeling multi-agent interactions with a fixed number of agents. Formally, a POSG is defined as a tuple $ G = \langle \mathcal{I}, \mathcal{S}, \mathcal{A}, \mathcal{O}, \mathcal{T}, \mathcal{Z}, \mathcal{R}, \gamma \rangle $, where $\mathcal{I} = \{1 \ldots N\}$ is a finite set of agents; $\mathcal{S}$ is the set of states; \(\mathcal{A}_i\) is the set of actions for agent \(i\);  $\mathcal{O}_i$ is the set of observations of agent $i$ ; $\mathcal{T}$ is the transition function, mapping the current state and  joint action to a probability distribution over  next states; $\mathcal{Z}_i$ is the observation function from which agent $i$ draws its observations; $\mathcal{R}_i $ is the reward function for agent \(i\);  $\gamma\in  (0,1)$ is the discount factor.

 There exists a rich literature of solution concepts proposed for various kinds of games. The most widely studied solution concept  is perhaps the Nash equilibrium (NE) \citep{doi:10.1073/pnas.36.1.48}. 

Given a joint policy $\pi = (\pi^{1}, \pi^{2}, \ldots, \pi^{N})$, the state-value function for agent $i$ is defined as the expected discounted return obtained when all agents follow $\pi$:
\begin{equation}
V^i(\pi; s)
=
\mathbb{E}_{\pi}
\left[
\sum_{t=0}^{\infty} \gamma^t r^i(s_t, \pi(s_t))
\,\middle|\, s_0 = s
\right]
\end{equation}
A joint policy $\pi$ is a Nash equilibrium if, for every state $s \in \mathcal{S}$ and every agent $i \in \mathcal{I}$,
\begin{equation}
V^i(\pi; s)
\ge
V^i(\pi^{\prime i}, \pi^{-i}; s),
\qquad
\forall \pi^{\prime i},
\end{equation}
where $\pi^{\prime i}$ denotes any alternative policy available to agent $i$.

While NE as a solution concept has broad applicability, in sequential and dynamic settings, it suffers from limitations such as failing to ensure rationality at every decision point as well as allowing non-credible threats, i.e., a strategic action that a player claims they will take but which they would have no incentive to follow through if the relevant decision point were actually reached. To account for these limitations, solution concepts such as subgame-perfect Nash equilibria (SPNE) \citep{10.1007/BF01766400} have been proposed in the game theory literature. SPNE refines NE by requiring that the players' strategies form a NE in every subgame of the original game, where a subgame is any subtree of the original game that starts at a single decision node and contains all of its successors. 

The problem of coordination among groups of agents has received considerable attention \citep{Boutilier, lauer2000algorithm, matignon2007hysteretic, Boutilier_bayesian}. Research has explored  learning communication protocols to improve coordination \citep{sukhbaatar2016learning, foerster2016learning}, as well as the emergence of natural language in such communication \citep{emergence, mordatch2018emergencegroundedcompositionallanguage}. In cases of no communication, agents have to be deployed independently and achieve coordination implicitly. 
MARL algorithms differ through the amount and type of communication that the agents are allowed among themselves. 
Algorithms belonging to the \textit{decentralized training decentralized execution} (DTDE) paradigm operate by having each agent consider all other agents as part of the environment; in this case, each agent trains independently. Despite their simplicity, DTDE algorithms such as Independent-Q learning (IQL) \citep{tan1993multi} and Proximal Policy Optimization (PPO) \citep{yu2022surprising} have been shown to perform competitively on several MARL tasks \citep{papoudakis2021benchmarking}.
One popular approach is to allow agents to access joint information in the training phase, but be completely independent during execution. These approaches fall under the category of \textit{centralized training decentralized execution} (CTDE)  \citep{Oliehoek_2008, CTDE}. Value-based CTDE methods factor the joint value function into decentralized components \citep{sunehag2017VDN, Qmix}. A specific implementation of this idea is Value Decomposition Networks (VDN) \citep{sunehag2017VDN}, in which the value function is decomposed into additive components, each of which is used by an individual agent. There also exists a class of algorithms that employ a centralized critic while maintaining decentralized actors. These methods, such as Multi-Agent PPO (MAPPO), leverage centralized information during training to improve coordination, yet execute policies independently at test time—embodying the CTDE framework.

\section{Fluid-Agent RL}
We begin by detailing a game model that allows agents to create other agents. Subsequently, we analyze existing solution concepts and their application to the proposed fluid-agent model.

\subsection{Partially Observable Fluid Stochastic Game}
To model interactions in a fluid-agent environment, we modify the POSG framework and define \textit{Partially Observable Fluid Stochastic Games} (POFSG)  $ G = \langle \mathcal{I}, \mathcal{S},\mathcal{L}, \mathcal{A}, \mathcal{O}, \mathcal{T}, \mathcal{Z}, \mathcal{R}, \gamma \rangle $, where \begin{itemize}
    \item $\mathcal{I} = \{1 \ldots N \}$ is the set of all agents that could be present in the environment at any time
    \item $\mathcal{S}$ is the set of states
    \item $\mathcal{L}: \mathcal{S} \rightarrow 2^\mathcal{I} $ is the \textit{alive} function, which returns the set of all agents present in the environment at a given state
    \item $\mathcal{A} = \prod_{i \in \mathcal{L}(s)} \mathcal{A}_i $  is the set of joint actions for all agents which are alive at state $s$; note that each agent's action space has a special ``spawning" action, which allows it to add to the number of alive agents 
    \item $\mathcal{O}_i$ is the set of observations of agent $i$
    \item $\mathcal{T}: \mathcal{S} \times \mathcal{A} \rightarrow Dist(\mathcal{S})$ is the transition function, mapping the current state and  joint action to a probability distribution over next states
    \item $\mathcal{Z}_i: \mathcal{S}  \rightarrow Dist(\mathcal{O}_i$) is the observation distribution of agent $i$
    \item $\mathcal{R}_i: \mathcal{S} \times \mathcal{A} \rightarrow \mathbb{R} $ is the reward function for agent \(i\)
    \item $\gamma\in(0,1)$ is the discount factor
\end{itemize}

We denote the population size at state $s$ by $|\mathcal{L}(s)|$. 
From the definition of $\mathcal{I}$, it follows that the maximum number of agents that could be present in a fluid environment is finite. This population ceiling is enforced when an alive agent chooses to spawn in state $s_t$, resulting in two possible scenarios. If $|\mathcal{L}(s_t)| = N$, then the action of spawning does not change the state. 
If $|\mathcal{L}(s_t)| < N$, another agent is added to the population.\footnote{In our experiments, this is done by choosing the agent with the smallest ID that is not alive at the moment and adding it to the set of alive agents.
}

We note that the POFSG definition is related to POSG, but with two important differences:
(a) the number of alive agents varies over time, which means that the joint action set varies over time as well; (b) one of the actions is to {\em spawn} another agent, which modifies {\em directly the action set}, but also modifies {\em indirectly  the set of reachable states} in the environment. Therefore, spawning can have an interesting effect on the optimal and/or equilibrium solutions of a problem. 

Even so, the Markovian nature of transitions, observations and rewards is common in our setup and POSG. The effect of the agent spawning action is also Markovian, in the sense that the new number of alive agents only depends on the previous one. Therefore, it is natural to consider applying existing POSG solution concepts and algorithms to this setting.

One important difference in this scenario compared to POSG arises if we consider adding fluidity to the fully cooperative case of POSG, in which agents share the same reward. In this case, spawning new agents changes the potential upper bound on performance that can be achieved. For example, if more agents are needed to solve a problem, a competent learning algorithm can learn to spawn more agents. However, the fully cooperative nature of the problem may change. For example, if agents spawn too many other agents in a fixed-resource environment, the resources can get depleted. Or, if the total reward is divided equally among agents, the amount per agent will vary during the game, which renders the game not fully cooperative, as more agents can lead to less reward per agent. Thus, the nature of optimal solutions and equilibria in the fluid-agent setting can change compared to the usual POSG/MARL setup.

\begin{figure*}
    \centering
    
    \begin{subfigure}[c][4cm][c]{0.25\textwidth}
        \centering
        \includegraphics[width=\linewidth,height=4cm,keepaspectratio]{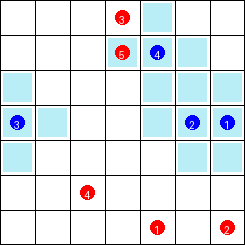}
        \caption{Predator-Prey}
        \label{predator_prey}
    \end{subfigure}
    \hspace{3em}
    \begin{subfigure}[c][4cm][c]{0.25\textwidth}
        \centering
        \includegraphics[width=\linewidth,height=4cm,keepaspectratio]{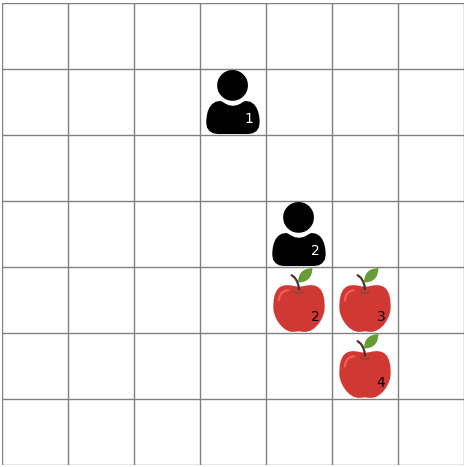}
        \caption{Level-Based Foraging}
        \label{level_based}
    \end{subfigure}
    \hspace{3em}
    \begin{subfigure}[c][4cm][c]{0.25\textwidth}
        \centering
        \includegraphics[width=\linewidth,height=4cm,keepaspectratio]{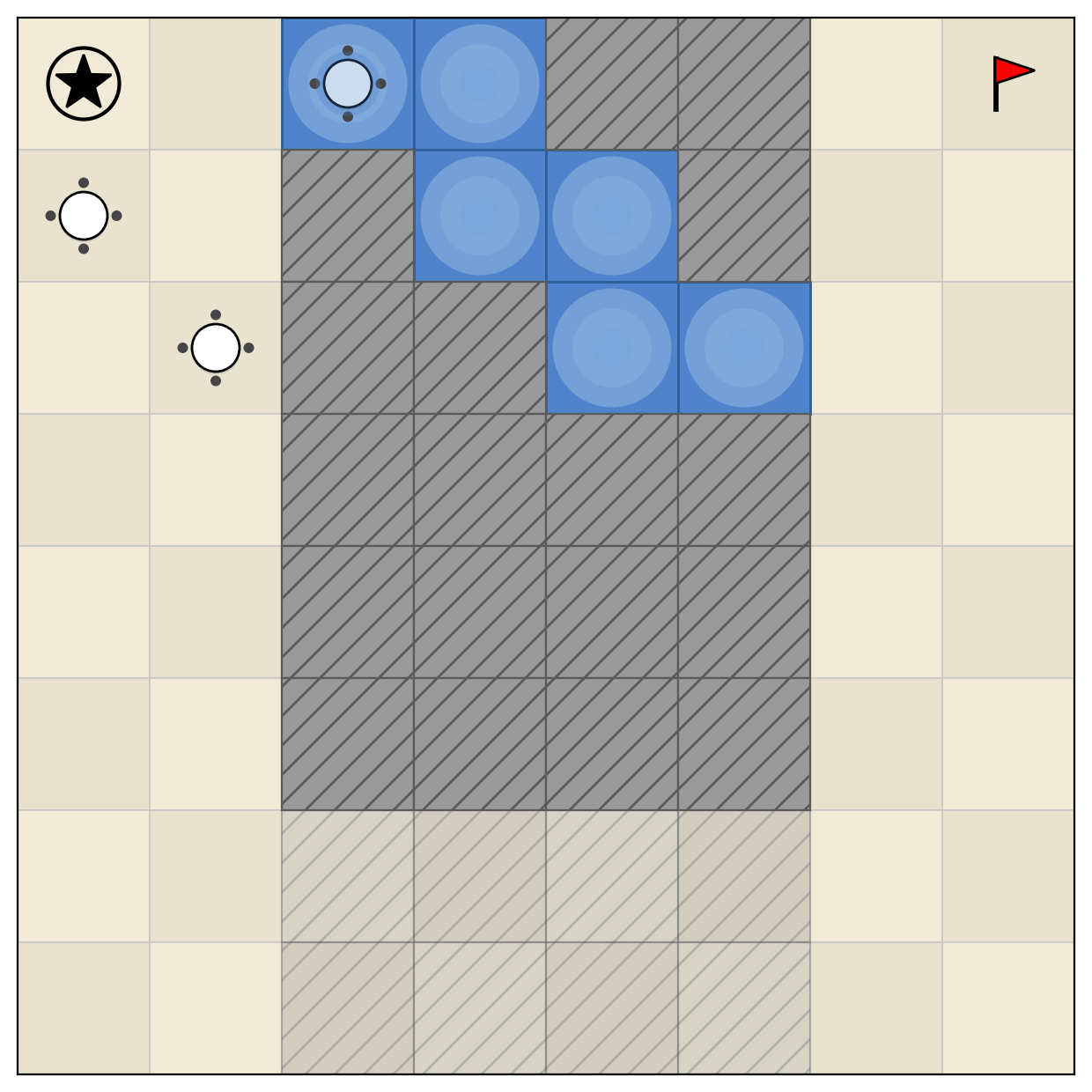}
        \caption{PuddleBridge}
        \label{puddle_bridge}
    \end{subfigure}
    \vspace{1em}
    \caption{Fluid-agent Environments}
    \label{envs}
\end{figure*}

\subsection{Nash Equilibrium}
 
Given the parallels between the proposed POFSG framework and existing POSG framework, we now discuss whether existing solution concepts for POSG can be leveraged to analyze fluid-agent games.
First, we prove the existence of a Nash equilibrium in a POFSG, by exploiting the flexibility of Fink's proof \citep{10.32917/hmj/1206139508} w.r.t allowing action spaces of agents to not only be variable, but to arbitrarily depend on the environment state. This allows us to transform a POFSG into a fixed-player stochastic game, with the unborn agents simply having a ``Do Nothing (N)" action. Thus, spawning can be viewed as an action that transitions the environment into a state with a modified action space for spawned agents.

\begin{theorem}[Existence of Stationary Nash Equilibrium]
\label{thm:ne}
Every POFSG possesses a stationary mixed–strategy Nash
equilibrium.
\end{theorem}

\begin{proof}
By Lemma~\ref{lem:embedding}, the game is an $N$–player discounted stochastic
game with finite players, state and action spaces.  The fixed–point argument of
Fink \citep{10.32917/hmj/1206139508} yields a stationary equilibrium in the augmented game.
Restricting the equilibrium strategies of the augmented game to the set of agents 
alive in each state yields a Nash equilibrium for the original fluid game, i.e.,
\[
\pi^{*}_{\text{fluid}}(s) = \{\, \pi^{*}_i(s) \mid i \in \mathcal{L}(s) \,\},
\]
where \( \pi^{*}_i(s) \) denotes the equilibrium mixed strategy of agent \( i \) in the  augmented game, and \( \mathcal{L}(s) \subseteq N \) is the set of agents alive in state \( s \).
\end{proof}

\subsection{Subgame--Perfect Nash Equilibrium}

\begin{theorem}[Existence of Subgame-Perfect Nash Equilibrium]
\label{thm:spne}
Every finite-horizon POFSG with publicly observed joint actions and perfect recall possesses a subgame-perfect Nash equilibrium.
\end{theorem}

\begin{proof}
We prove the theorem in three steps.

\noindent\textbf{Step 1: Sequentialization.}
Consider a stage \(t\) with state \(s_t\) and alive agents \(\mathcal{L}(s_t)\).
Transform the simultaneous-move stage into a sequential-move extensive form by fixing a common-knowledge ordering of agents (e.g., by increasing ID).
Let \(G^{\mathrm{seq}}\) denote the resulting extensive-form game, whose decision nodes are tuples \((t,s_t,i)\) indicating stage \(t\), state \(s_t\), and the turn of agent \(i\).
While the assumption of observed joint actions allows the information set at any decision node to contain information of all agents' moves in the previous stages, agents remain unaware of any moves made by other agents in the current stage, hence preserving the information structure of the simultaneous-move game.

\noindent\textbf{Step 2: Perfect recall and subgame definition.}
We assume agents have \emph{perfect recall}: each agent remembers its own past actions and private observations.
Together with public monitoring of joint actions, this implies that in \(G^{\mathrm{seq}}\) every agent knows both the complete public history and its own private history at each of its decision points.
Consequently, \(G^{\mathrm{seq}}\) is a finite-horizon extensive-form game with perfect recall.

A \emph{stage-start subgame} is rooted at a node \((t,s_t,i_1)\), where \(i_1\) is the first mover of stage \(t\).
Since the public history is common knowledge and private histories are known to the respective agents, this node belongs to a singleton information set.
Thus every stage-start node defines a proper subgame of \(G^{\mathrm{seq}}\).

\noindent\textbf{Step 3: Existence and lifting to the original POFSG:}
The game \(G^{\mathrm{seq}}\) is finite (finite horizon, bounded population, finite action sets).
Because $G^{\mathrm{seq}}$ is a finite extensive-form game, every proper
subgame is itself a finite extensive-form game. By Nash's theorem
\cite{540b73bd-a3f1-333e-a206-c24d0fbbb8bc}, every finite game admits a mixed-strategy Nash
equilibrium. Applying this result to each subgame and proceeding by
backward induction over the finite horizon yields a mixed-strategy
subgame-perfect Nash equilibrium of $G^{\mathrm{seq}}$. As \(G^{\mathrm{seq}}\) has perfect recall, Kuhn's theorem \citep{Kuhn+1953+193+216} guarantees the existence of an \emph{outcome-equivalent} behavioral-strategy SPNE.

Finally, we map this equilibrium back to the original POFSG.
In \(G^{\mathrm{seq}}\), an agent's strategy depends only on information that is available in the original POFSG by construction.
Hence the behavioral strategy of \(G^{\mathrm{seq}}\) induces a well-defined policy for the simultaneous-move POFSG.
Because the equilibrium is subgame-perfect in $G^{\mathrm{seq}}$, it is sequentially rational at every information set of the game. The induced behavioral strategy profile therefore constitutes a subgame-perfect Nash equilibrium of the original POFSG.

\end{proof}

\section{Fluid-agent Environments}
\subsection{Fluid Predator-Prey}
We work with a predator-prey class of environments \citep{magym}, an example of which is shown in Fig.~\ref{predator_prey}. 

\subsubsection{Dynamics} The environment consists of a square grid, with predators depicted as blue circles and preys as red. Blue squares in the cardinal direction of each predator denote their range for capturing prey. However, a prey must be in range of two different predators in order to be captured; thus, the predators must cooperate. At any time, a square can only be occupied by a single entity. Movement by either a predator or a prey is therefore only possible if the target square is empty. The preys follow a fixed policy of either staying at the same position, with probability 0.3, or randomly moving in any direction with probability 0.175, unless the randomly selected choice would move it into range of a predator. In this case, the prey simply stays still. Once the team of agents takes an action, the target position for each agent, starting with the agent with ID 1, is checked and if empty, the agent is moved. After all agents are at their new positions, any preys now within range of at least 2 agents are captured and removed from the grid. Finally, the preys that are still alive are moved to their new positions. An episode ends when either all the preys are captured or when 100 time steps have elapsed.

\subsubsection{Observation Space}
We restrict the visual field of each agent to an $11 \times 11$ window centered around it, as in previous works \citep{tan1993multi}. The agents observe their own coordinates as well as their ID, the latter being necessary for performing parameter sharing \citep{Gupta2017CooperativeMC}. While these three components might suffice for solving the Predator-Prey environment, adding spawning dynamics to the environment makes it much harder to solve. Consequently, in order for the agents to learn to spawn, we augment the observation space of each agent with the team's joint action taken in the previous step, a \emph{parent indicator} i.e. the number of children the agent itself has spawned, and the number of preys and agents present in the environment. The motivation behind conditioning the agents on the previous joint action and the number of agents spawned is to keep the credit assignment problem tractable. The information about the number of preys and agents present is added as it is crucial for fluid-agent groups to accurately judge whether spawning is useful. 

\subsubsection{Action Space}

Each agent's action space consists of actions \texttt{move[each direction]}, \texttt{do nothing}, and \texttt{spawn}. After a spawning action, new agents are spawned randomly at an empty  position. There are no restrictions on how many agents can be spawned in one step, but the total number of agents is bounded.

\subsubsection{Reward Function}
We retain a common-reward structure i.e. all alive predators receive the same per-agent reward. The design choice is whether to normalize the payoff for capturing prey by the current population size. Let $P$ be the prey\_capture\_reward, $c_{\mathrm{spawn}}$ the spawn penalty, $c_{\mathrm{step}}$ the per-step cost per alive agent, $N_{\mathrm{cap}}(s,a)\in\mathbb{N}_0$ the number of captures on $(s,a)$, and $N_{\mathrm{sp}}(s,a)\in\mathbb{N}_0$ the number of successful spawns. Each alive agent $i\in \mathcal{L}(s)$ receives
\[
r_i(s,a)\;=\;\underbrace{\mathrm{Payoff}(s,a;\mathcal{L}(s))}_{\text{predation reward}}
\;-\;\underbrace{\Big(\tfrac{c_{\mathrm{spawn}}\,N_{\mathrm{sp}}(s,a)}{|\mathcal{L}(s)|}+c_{\mathrm{step}}\Big)}_{\text{costs (spawn split per capita; step per agent)}},
\]
where $\mathcal{L}(s)$ is measured pre-transition.

\textit{Size-constant payoff (SCP)} does not change with population size:
\[
\mathrm{Payoff}(s,a;\mathcal{L}(s))=P\,N_{\mathrm{cap}}(s,a)
\]
Here, the per-agent capture reward is invariant to $|\mathcal{L}(s)|$, while the joint capture reward scales linearly with the number of agents.

\textit{Size-inverse payoff (SIP)}
performs a normalization per capita:
\[
\mathrm{Payoff}(s,a;\mathcal{L}(s))=\frac{P\,N_{\mathrm{cap}}(s,a)}{|\mathcal{L}(s)|}
\]
Here, the \emph{per-agent} capture payoff decreases with $|\mathcal{L}(s)|$, while the \emph{joint} capture payoff remains invariant (summing identical per-agent rewards over $|\mathcal{L}(s)|$ recovers $P\,N_{\mathrm{cap}}$). Spawn cost is per-capita in both variants; $c_{\mathrm{step}}$ is a per-alive linear pressure in both.

\subsection{Fluid Level-Based Foraging}
We adapt the Level-Based Foraging (LBF) environment \citep{albrecht2015gametheoreticmodelbestresponselearning}, illustrated in Fig.\ref{level_based}, to a fluid paradigm. 

\subsubsection{Dynamics} The world is a square grid populated by food items (red apples labeled with their level) and agents (black silhouettes labeled with their levels).  Agents can only attempt to collect food in the four neighboring cells. Any collection attempt of a level $\ell$ food is only successful if the sum of levels of surrounding agents is $\geq \ell$. 

\subsubsection{Observation Space}

Unlike Predator-Prey, here we consider the fully observed setting. Each agent receives an observation, encoding the positions and levels of all food items 
\((x, y, \ell)\), and for each agent, the positions, levels, alive flags, parent indicators, and previous actions 
\((x, y, \ell, \text{alive}, \text{parent}, a_{\text{prev}})\).
Two additional scalars record the environment’s population capacity and the agent’s own ID. 
This yields an observation of dimension, $
\text{obs\_dim} = 3N_{\text{food}} + 6N_{\text{agents}} + 2,
 $ allowing agents to reason jointly over food distribution, team composition, and population constraints 
when deciding whether to spawn or act cooperatively.

\subsubsection{Action Space}
Originally, the action space of each agent consisted of actions \texttt{do nothing}, \texttt{move[each direction]} and \texttt{load}. In the fluid version, we also include a \texttt{spawn} action that lets any alive agent spawn another agent at a random empty location in the grid. The spawned agent inherits the level of its parent agent. Every square can hold at most one entity, and movement (N, S, E, W) succeeds only when the destination cell is empty. As with our Predator-Prey variant, entity updates are executed sequentially in order of increasing agent ID, such that collision resolution is deterministic.

\subsubsection{Reward Function}
Loading a food of level $\ell$ results in each agent receiving a share of $\ell$ weighted by their level. As in Predator-Prey,  the spawn cost $c_s$ is shared among all alive agents. However, agents that do not participate in collecting a food do not receive any reward. Episodes terminate when all foods are consumed or after 100 steps. A step cost $c_{\text{step}}$ applies to all alive agents.

\subsection{PuddleBridge}

We introduce the \emph{PuddleBridge} environment (Fig.~\ref{puddle_bridge}), an $8\times8$ grid world featuring land (checkered blocks), walls (grey blocks with diagonal lines), puddles (blue blocks with circles), a designated spawn cell (represented by $\star $), and a goal cell (represented by \tikz[baseline=0.1ex]{\draw[thick] (0,0) -- (0,0.3);\fill[red] (0,0.3)--(0.25,0.2)--(0,0.1)--cycle;}). Agents (represented by white circles) begin in the spawn cell and must coordinate to reach the goal  on the opposite side of a central wall barrier.  The land, spawn cell and goal are single-occupancy; Walls are impassable. 
  
\subsubsection{Puddle Dynamics}
\begin{itemize}
\item \textbf{Capacity.} Entering a puddle is permitted iff its current occupancy is $<2$. 
\item \textbf{Stacking.} If an agent moves into a puddle cell already holding one agent, given that the already present agent stays in the same position, the recently moved agent becomes the \emph{top}; the cell is then a 2-stack.
\item \textbf{Bottom lock.} While a cell is a 2-stack, the bottom agent is immobilized for that entire step and can only spawn; the top agent may move and spawn.\footnote{This is evaluated from the pre-move state each step; once a cell is a 2-stack at step start, the bottom remains locked for that step even if the top leaves earlier in the move ordering.}
\item \textbf{Puddle $\to$ puddle mobility.} A puddle $\to$ puddle move is permitted \emph{only} for the \emph{top} agent of a 2-stack. In that case, the move is treated like a land $\to$ puddle move: any destination puddle with free capacity is admissible (it need not already be occupied). A lone occupant of a puddle cell cannot perform a puddle $\to$ puddle move.
\item \textbf{Unstacking.} When a 2-stack loses its top (the only one allowed to leave), the cell reverts to single occupancy.
\end{itemize}

The lightly shaded blocks in Fig. \ref{puddle_bridge} represent walls that may or may not be present in each episode, simulating a gating mechanism. When this group of walls is absent, i.e. the gate is open, the agent can reach the goal state by itself, by going across the central wall barrier. If the walls are present, i.e. the gate is closed, this route is blocked. There is an alternative route through the puddles, but it requires coordination with another agent 
to move from puddle $\to$ puddle. 
The agent stacking mechanic enables a physical ``bridge'' over the puddle, where the top agent can move onward while the bottom agent remains submerged, effectively turning local cooperation into a spatial resource.

\subsubsection{Observation Space}
For every grid cell, each agent's observation includes a one-hot encoding of the tile type and a scalar ID feature representing the base occupant (bottom agent if stacked). Each puddle cell contributes an additional virtual slot recording the top agent’s ID. Agent-specific information is appended at the end: the agent’s normalized position $(r,c)$, its own ID, and a one-hot vector of previous actions for every agent in the population.

\subsubsection{Action Space}
Each agent selects one of six discrete actions: \texttt{none}, \texttt{north}, \texttt{south}, \texttt{west}, \texttt{east}, or \texttt{spawn}. The \texttt{spawn} action attempts to create a new agent at the spawn cell if it is unoccupied and the episode’s population cap is not yet reached. Agents act in ascending order of IDs to ensure deterministic collision resolution.

\subsubsection{Reward Function}
Reaching the goal yields a team reward of $10$, shared equally among all currently alive agents. Each alive agent pays a per-step cost $c_{\text{step}}$, and successful spawns incur a team spawn cost $c_{\text{spawn}}$, equally divided among agents alive at the start of the step. Episodes terminate when any agent reaches the goal or after $100$ steps.

\subsection{Exploration Paradigm}
Introducing fluidity exacerbates the challenge of exploration. Beyond exploring the environment itself, agents must now also explore how learned behaviors interact with different population sizes, as these alter the environment’s dynamics. Naïve exploration strategies often become trapped in suboptimal extremes of population size. To address this, during training, we randomly sample both the initial population and a population ceiling (less than or equal to the true ceiling) at the start of each episode. This forces agents to experience and adapt to a wider range of population sizes that would otherwise be underexplored. For all fluid-agent environments discussed thus far, we additionally include this sampled ceiling in the agents’ observation space. During evaluation, the sampled ceiling is replaced with the true population limit. While randomizing the initial population is straightforward in Predator–Prey, owing to the complete homogeneity among agents, the process is more involved in Level-Based Foraging (LBF) and PuddleBridge. In LBF, we handle this by sampling both population size and agent levels randomly at initialization, constraining the sampled levels to the maximum level among the agents initially present in the episode. In PuddleBridge, where spatial relationships are critical, any agents beyond the first are placed only in cells adjacent to the spawn cell, ensuring 
consistent starting configurations.

Moreover, for algorithms employing simple exploration schemes such as $\epsilon$-greedy (e.g., IQL and VDN), we introduce a separate exploration parameter $\epsilon_{spawn}$ for spawning, which increases linearly throughout training. When randomly choosing an action, we choose \texttt{spawn} with probability $\epsilon_{spawn}$ and other actions with probability $(1 - \epsilon_{spawn})/(|\mathcal{A}| - 1)$. This approach, inspired by curriculum learning, allows agents to first learn effective behaviors at smaller population sizes before progressively adapting their policies to settings with larger and more variable populations.

\section{Empirical Analysis}

\subsection{Experimental Design}

This section outlines the design and rationale behind our experimental setup. 
We design environments and evaluation protocols that allow us to probe specific questions about how fluidity interacts with learning dynamics, coordination structure, and policy optimality.

\subsubsection{Algorithmic Inductive bias vs. Fluid reward}

Rather than adopting the conventional classification of algorithms into CTDE and DTDE, we instead categorize them by their optimization target i.e. whether they optimize the per-agent or joint return. This distinction is particularly salient under fluidity, since unlike in static settings where the two objectives coincide for cooperative agents, fluid-agent environments introduce an additional means of improving the joint return: adapting the team size itself. This distinction is illustrated by the case of a size-constant payoff: agents may achieve an optimal per-agent return yet fail to maximize the joint return if they neglect to spawn additional agents whose inclusion would yield a net positive contribution to the joint return. To examine how an algorithm’s optimization target interacts with the reward structure under fluidity, we evaluate performance under two distinct payoff formulations. With a size-constant payoff, algorithms optimizing the joint return are incentivized to spawn additional agents as a means of increasing total reward, whereas per-agent optimizers lack such incentive unless the spawned agents contribute a net positive return to their parent. The dynamics become more balanced under a size-inverse payoff, where spawning provides no direct advantage in joint return. In this case, the only incentive to spawn is to increase the average per-agent return experienced by the alive agents. We conduct experiments using IQL, VDN, PPO, and MAPPO, evaluating the latter under two critic configurations: one with concatenated observations as input and another using the global state (the full predator–prey grid), which we denote as MAPPO\_state.

\subsubsection{Task Variability}

A key advantage expected of fluid agents over a fixed group is their ability to dynamically adapt to task variability. To test this hypothesis, we conduct an ablation study in the Predator–Prey environment, where the number of preys is randomly sampled from a multinomial distribution at the start of each episode. With a substantial spawn cost, a fluid group should learn to limit spawning when preys are scarce and expand when they are abundant. Since fixed groups incur no spawn cost, we subtract these costs post hoc for fairness. Although this adjustment introduces different learning dynamics, the fluid group nonetheless faces a significantly harder problem—it must learn not only to capture preys efficiently but also to optimize its own population size.

\subsubsection{Adaptive Team Composition}
We next examine whether agents can learn to optimize team composition, deciding not only whether to spawn but which type of agent to add. Because all spawned agents in Predator–Prey are indistinguishable, we turn to LBF, where we introduce an inherited parent level rule—each child inherits the level of its parent. We then construct a scenario with two initial agents of levels \{1,2\}, food items of levels \{2,3,4,5\}, and a maximum of four agents. With a fixed spawning cost per agent, the optimal policy is to spawn exactly one additional agent, specifically of level 2, since this configuration alone enables the team to collect all four foods while paying the smallest total spawn cost. We set the spawn and step costs to $c_{\text{spawn}} = 1.0$ and $c_{\text{step}} = 0.025$.

\subsubsection{Emergence of Richer Policies through Fluidity}
Up to this point, our analysis has assumed that spawned agents are policy clones of their parent, reflecting parameter sharing among all agents. We now relax this assumption and consider the case in which agents possess distinct policies. This raises several natural questions. Can an agent learn not only when to spawn, but also how to coordinate effectively with a heterogeneous teammate? Can it dynamically alternate between fluid and fixed strategies depending on context? More broadly, does fluidity lead to qualitatively different behaviors compared to fixed-agent settings?
To investigate these questions, we conduct experiments in the PuddleBridge environment, simulating agent independence by disabling parameter sharing and analyzing the resulting behavior of VDN algorithm. The spawn and step costs are fixed at $c_{\text{spawn}} = 1.0$ and $c_{\text{step}} = 0.1$.

\subsection{Evaluation}
For all algorithmic comparisons, we report the normalized, undiscounted joint return unless specified otherwise. This is done by computing the joint return over 1000 test episodes evaluated at 100 uniformly spaced checkpoints during training and normalizing using the following formula: $R_{\text{norm}} = (R - R_{\min})/(R_{\max} - R_{\min})$
where \( R \) denotes the observed joint return, and \( R_{\min} \) and \( R_{\max} \) are the minimum and maximum returns across all algorithms and timesteps, respectively. All results are averaged over 5 seeds, with shaded area denoting one standard deviation. 

\subsection{Results}
\begin{figure*}[h]
\centering
     \includegraphics[width=0.8\linewidth]{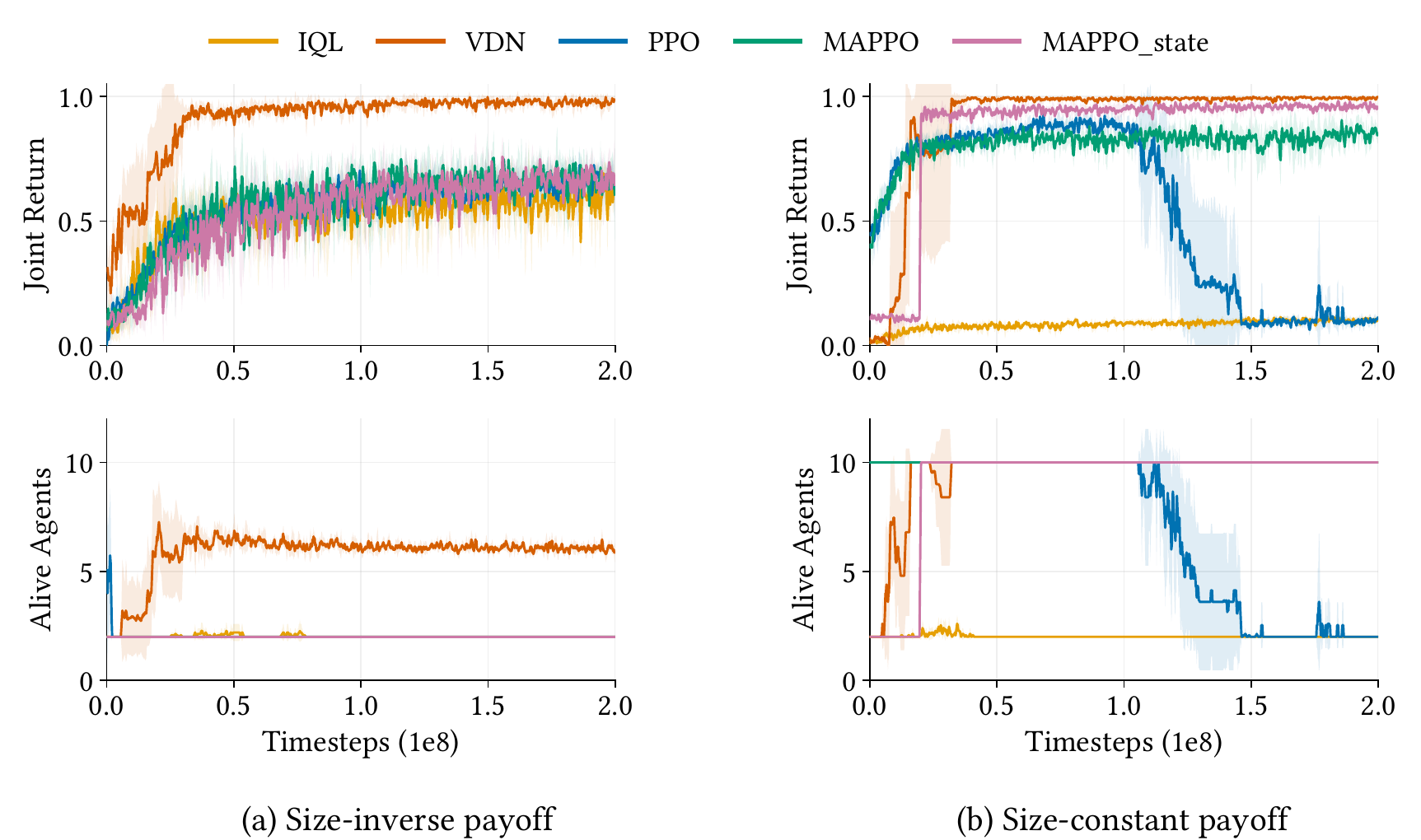}            
    \caption{Algorithmic comparison on Predator-Prey environment under different payoff normalizations.} 
    \label{RewScale}
\end{figure*}

\subsubsection{Algorithmic Comparison}

To compare the nature of different MARL algorithms in the fluid-agent setting, we run experiments with a $21\times21$ grid with 42 preys, prey\_capture\_reward $P=5$ and $c_{\text{step}}=0.01$. For all experiments on Predator-Prey, the fluid group of agents begins with 2 agents and is restricted to a maximum population of 10 agents.\begin{itemize}
    \item \textbf{SIP:} We first consider the size-inverse payoff, with $c_{\text{spawn}}=10$. Fig.~\ref{RewScale}(a) displays the normalized joint return and the agents alive at the end of each test episode. As expected, under a size-inverse payoff, per-agent optimization algorithms are discouraged from spawning, since each additional agent sharply reduces the per-agent return—even when doing so could improve the joint return. In contrast, value decomposition methods such as VDN, which optimize directly for the joint return, perform well in this setting. Notably, VDN does not spawn excessively; instead, it effectively balances the cost–benefit trade-off of adding agents and converges to spawning approximately six agents.
    
    \item \textbf{SCP:} For the size-constant payoff, we counterbalance the clear incentive to spawn—since any agent capturing a prey yields a shared reward P for all—by setting a relatively high spawn cost of $c_{\text{spawn}}=50$. Fig.~\ref{RewScale}(b) displays the normalized joint return and the agents alive at the end of each test episode. Despite this penalty, the optimal per-agent return still corresponds to spawning the maximum number of agents. Empirically, we observe that IQL and PPO fail to reach this optimum, whereas CTDE-based methods like VDN and MAPPO achieve it with ease.

\end{itemize}
\clearpage

\subsubsection{Ablation}
To perform this fluidity ablation, we use a \(25 \times 25\) grid and evaluate fixed groups of 
\(2, 4, 6, 8,\) and \(10\) agents alongside a fluid group, all trained using VDN. We treat \(2\) agents as the 
baseline population, meaning all runs begin with two agents, and fluid agents incur a spawn 
cost only when the population exceeds this baseline. For fairness, the spawn cost 
 \(c_{\text{spawn}} \times (n - 2)\), is applied \textit{post hoc} to each fixed group of $n$ agents. We maintain the prey\_capture\_reward $P=5$, $c_{\text{step}}=0.01$, $c_{\text{spawn}}=10$ and employ the size-inverse payoff for fluid agents. Preys are sampled from a multinomial distribution of \{20,40,60,80\} each episode. For each prey population, we measure the joint return over 1000 post-convergence episodes for all agent groups. 
\begin{figure*}[t]
\centering
\captionsetup[subfigure]{justification=centering,singlelinecheck=false}

\newlength{\figH}
\setlength{\figH}{6.2cm}
\captionsetup[subfigure]{skip=4pt}
\begin{subfigure}[t]{0.58\textwidth}
  \centering
  \hspace*{-22mm}%
  \begin{minipage}{\linewidth}
    \centering
    \includegraphics[height=\figH]{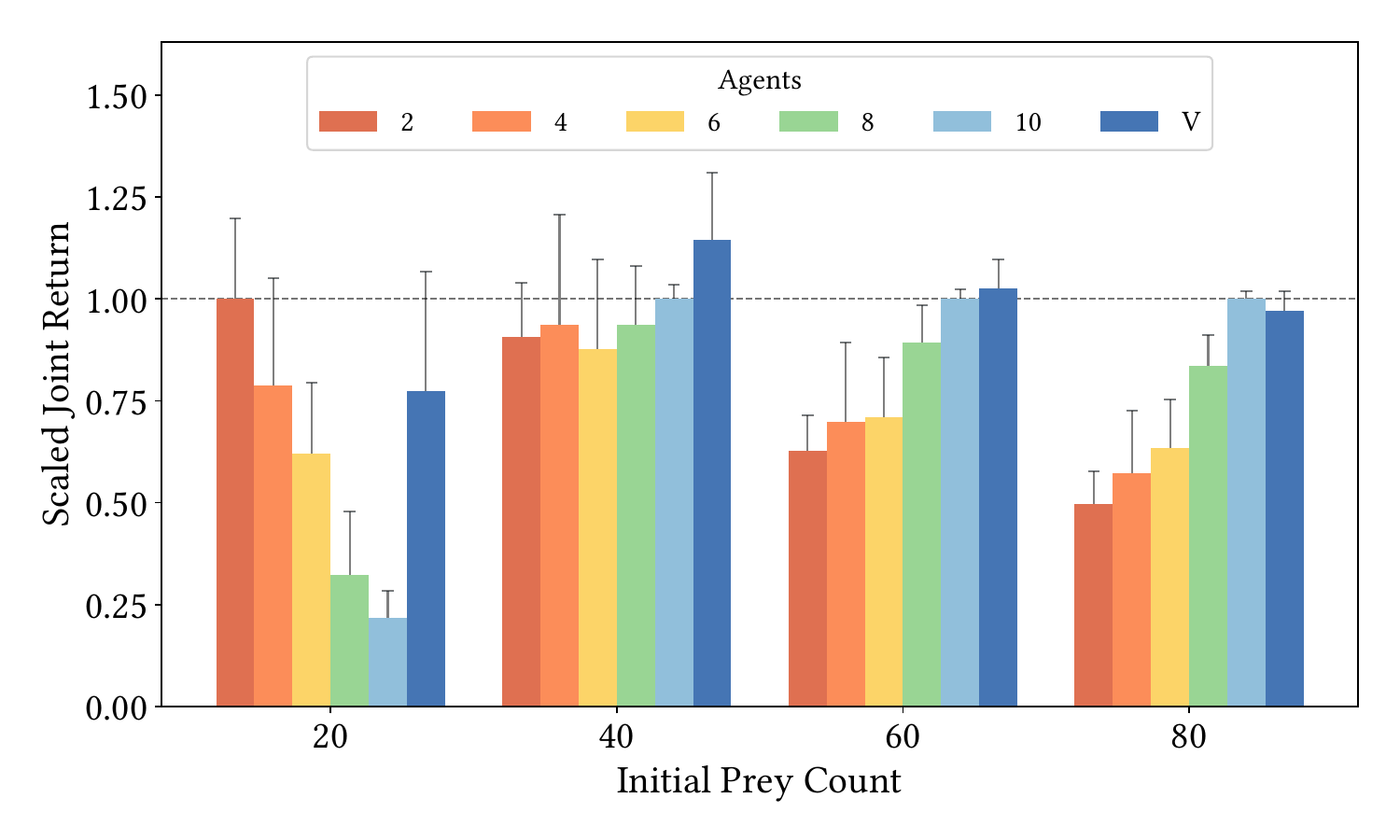}
    \subcaption*{\hspace{4.5em}(a)}    
  \end{minipage}
\end{subfigure}\hfill
\begin{subfigure}[t]{0.35\textwidth}
  \centering
  \hspace*{-18mm}%
  \begin{minipage}{\linewidth}
    \centering
    \includegraphics[height=\figH]{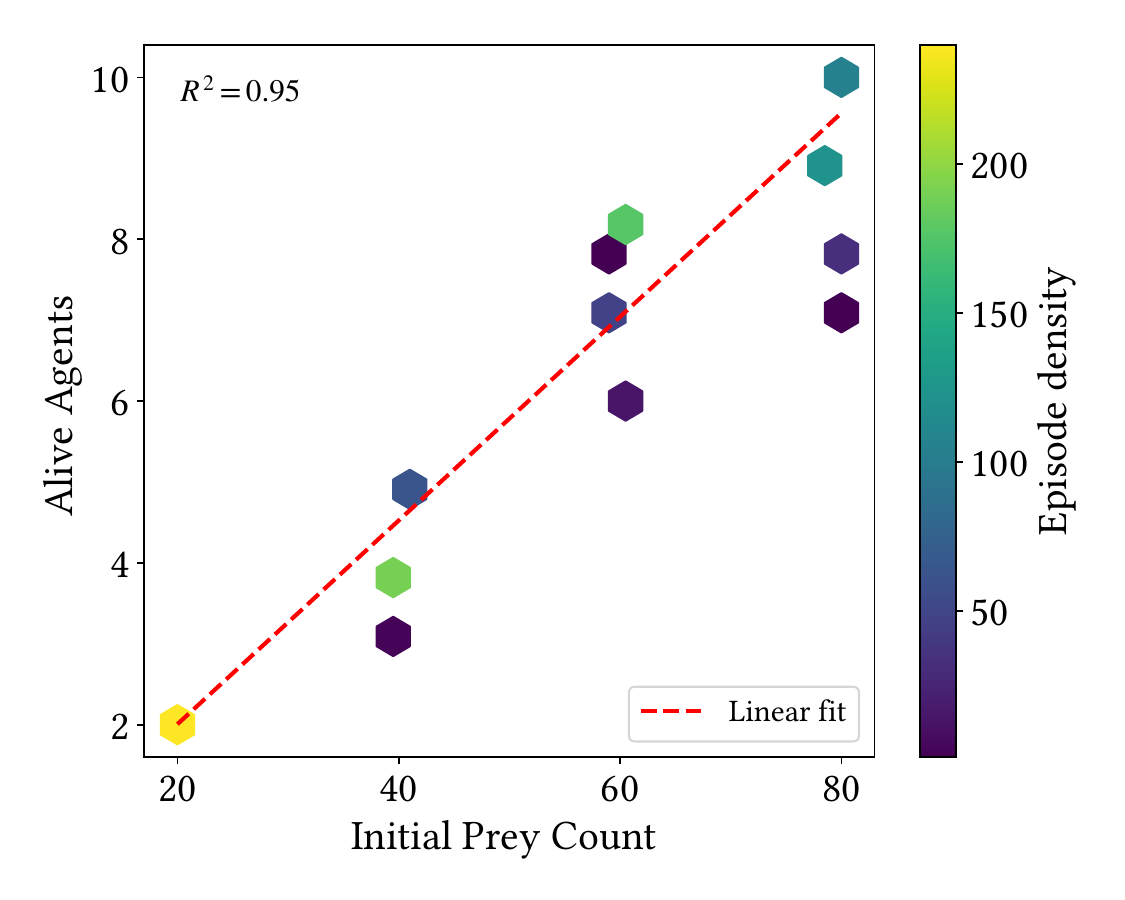}
    \subcaption*{\hspace{2.8em}(b)}    
  \end{minipage}
\end{subfigure}
\caption{(a) Fluid and fixed groups on tasks with preys sampled from a distribution. Error bars represent one standard deviation. (b) Density plot showing the relationship between episodic prey availability and converged fluid-agent populations.}
\label{fig:task_abl}
\end{figure*}

Fig.~\ref{fig:task_abl}(a) compares the joint returns achieved by fixed-size predator groups (2–10 agents) and a fluid population (“V”) across varying levels of prey availability. Returns are scaled by the maximum fixed-group performance to highlight relative efficiency. At low prey counts, small fixed groups (2–4 agents) achieve near-optimal returns due to minimal spawn and step costs. However, as prey abundance increases, their performance saturates, failing to capitalize on the additional resources. Larger fixed groups show the opposite trend: they underperform in low-resource settings due to higher per-step and spawn costs, but improve with prey richness. In contrast, the fluid group consistently adapts its population size to the resource level and performs competitively with the best fixed group across prey densities. Fig.~\ref{fig:task_abl}(b) depicts the relationship between initial prey availability and the population size adopted by the agents. Each hexagon represents the density of episodes, while the red dashed line traces the average equilibrium population attained across varying prey counts. This demonstrates that agents collectively learn to regulate their population size in response to environmental abundance—expanding when resources are plentiful and refraining from spawning when scarce.

\clearpage
\subsubsection{Level-Based Foraging}
Figure~\ref{fig:LBF_results} shows the normalized joint return and the number of level-1 and level-2 agents spawned at the end of each episode for different MARL algorithms. In the optimal configuration, agents should spawn exactly one additional agent, with the level-2 agent initiating the spawn so that the new teammate also has level 2. This setup enables capturing level-5 food with only three agents while minimizing spawn cost. From the plot of level-2 agents spawned, it is evident that PPO, MAPPO and VDN all learn the optimal strategy—spawning exactly one additional level-2 agent—while IQL approaches similar behavior with slightly higher variance. This confirms that fluid agents can optimize not only population size, but also the type of agent spawned, enabling precise and role-aware spawning strategies.

\begin{figure*}[t]    
     \centering
     \begin{subfigure}[c]{1\textwidth}                
     \includegraphics[width=\textwidth]{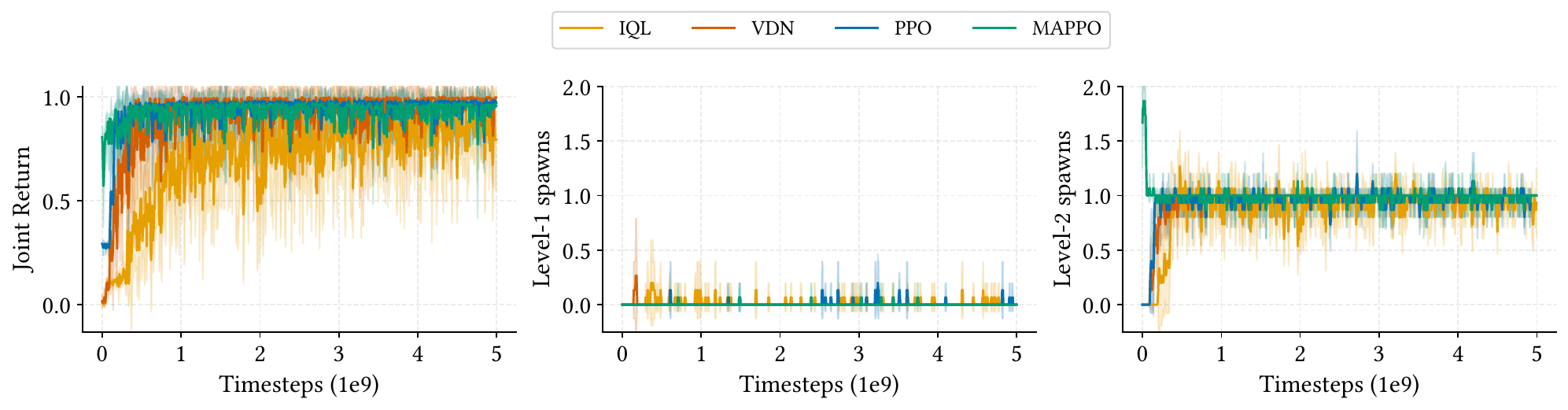}         
     \end{subfigure}                       
    \caption{Algorithmic comparison on the LBF environment.}
    \label{fig:LBF_results}
\end{figure*}

\vspace{2em}
\begin{figure*}[h]         
     \begin{subfigure}[c]{1\textwidth}         
     \centering
     \includegraphics[width=0.5
     \textwidth]{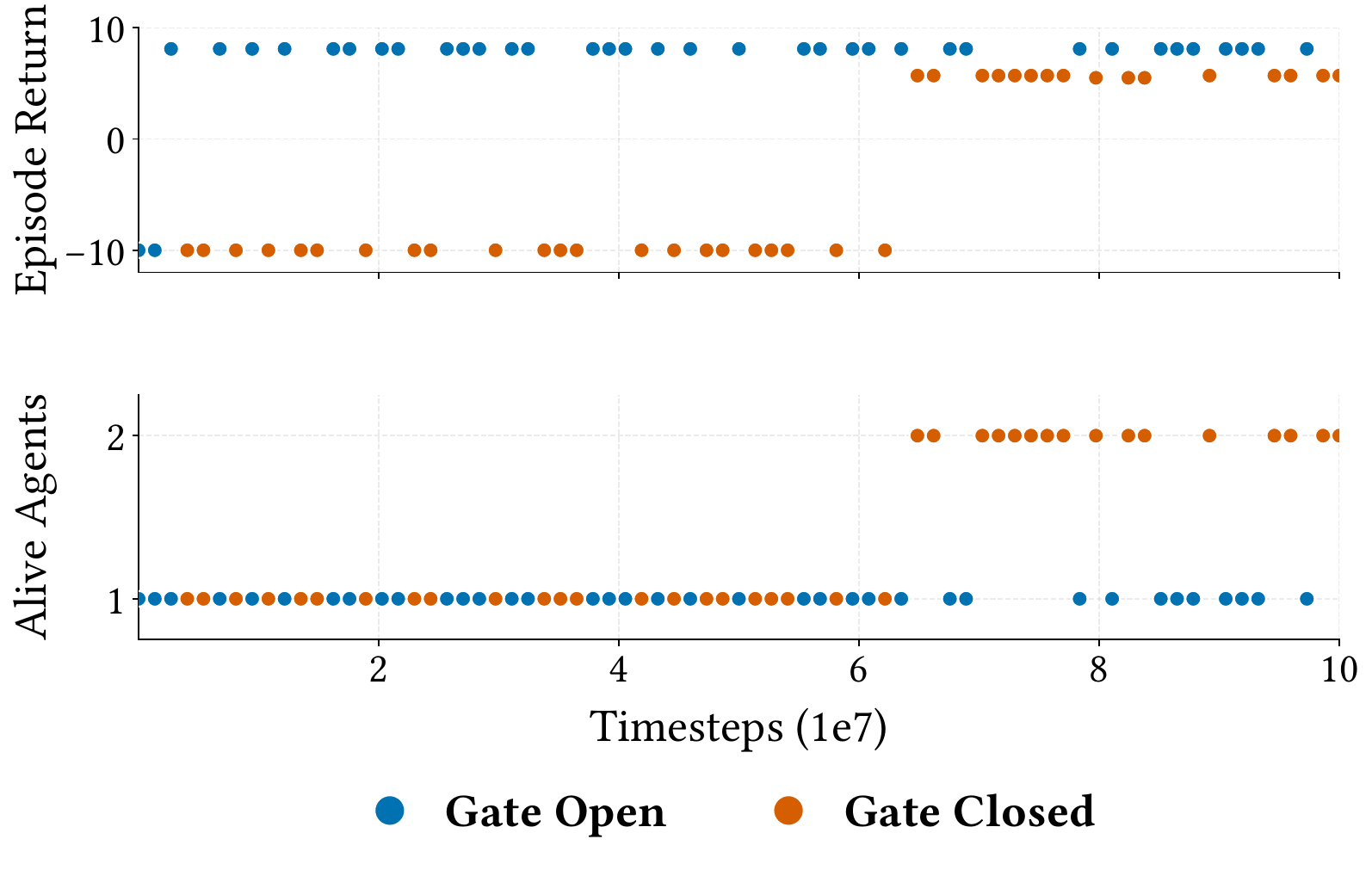}              
     \end{subfigure}               
        \caption{Episode return (top) and alive agents (bottom) during training on PuddleBridge. Points are colored by gate condition: open (blue) or closed (orange). }
        \label{PuddleBridgeResults}
\end{figure*}
\subsubsection{PuddleBridge}
    Figure~\ref{PuddleBridgeResults} presents 75 equidistantly sampled training checkpoints, reporting the corresponding unnormalized joint episode return (top) and the number of agents alive at the end of each episode (bottom), with the color of each data point indicating whether the gate was open (blue) or closed (orange). Because the policy for reaching the goal when the gate is open is straightforward, the agent quickly learns it, as reflected by the blue points corresponding to single-agent episodes with positive returns. In contrast, the closed-gate scenario requires both spawning and coordination, making it substantially harder to master. Nonetheless, the agents eventually learn this more complex policy. A by-product of spawning in this setting is that, while two agents are required to achieve a positive return when the gate is closed, the accumulated spawn and step costs make this policy suboptimal when the gate is open. The agents successfully learn this distinction—highlighted by the absence of blue points around the two-agent mark, adapting by switching between fluid and non-fluid policies to achieve optimal performance across both conditions.

\section{Conclusion and Future Work}
We propose Fluid-Agent Reinforcement Learning as a framework for modeling real-world interactions where the assumption of a fixed number of agents is overly restrictive. In this setting, a team of agents can dynamically expand its population through a spawning mechanism integrated into the environment. We evaluate the performance of different MARL algorithms under various reward structures and show that the fluid-agent setting enables teams to dynamically adjust their population size in response to task variability, optimize the characteristics of spawned agents and also discover richer policies that are inaccessible in fixed-agent settings.

This work primarily introduces the fluid-agent framework and focuses on the aspect of agent creation. Future directions include studying the more general case where environments support agent creation and death, as well as investigating the theoretical properties of fluid systems—such as conditions for equilibrium convergence and sample complexity under fluid dynamics.


\bibliographystyle{plainnat}
\bibliography{main}

\clearpage
\appendix

\section{Additional Theoretical Results}

\begin{lemma}[Fixed–Population Embedding]
\label{lem:embedding}
Consider a POFSG  $ G = \langle \mathcal{I}, \mathcal{S},\mathcal{L}, \mathcal{A}, \mathcal{O}, \mathcal{T}, \mathcal{Z}, \mathcal{R}, \gamma \rangle $. For every agent $i$ in $G$, we define an \emph{augmented} action space $\mathcal{A}^{+}_i(s)$ dependent on the state $s$:
\begin{equation}
\mathcal{A}^{+}_i(s)
  \;=\;
  \begin{cases}
    \mathcal{A}_i, & i\in \mathcal{L}(s),\\[2pt]
    \{\mathrm{N}\}, & i\notin \mathcal{L}(s).
  \end{cases}  \nonumber
\end{equation}

where $N$ denotes a dummy “do nothing” action.  This allows us to define the joint augmented action space as $ \mathcal{A}^{+}(s)=\prod_{i=1}^{N}\mathcal{A}^{+}_i(s)$. Next, define the \emph{augmented} transition kernel $\mathcal{T}^+: \mathcal{S} \times \mathcal{A}^{+} \rightarrow Dist(S)$:
\begin{equation}
    \mathcal{T}^{+}(s_{t+1}|s_t, \{a_i: i \in \mathcal{I}\}) = \mathcal{T}(s_{t+1}|s_t, \{a_i: i \in \mathcal{L}(s_t)\}) \nonumber 
\end{equation}
Finally, define the \emph{augmented} reward function $\mathcal{R}_i^+ : \mathcal{S} \times \mathcal{A}^{+} \rightarrow \mathbb{R}$ for each agent in the POFSG. 
\begin{equation}
\mathcal{R}_i^+(s, a)
  \;=\;
  \begin{cases}
    \mathcal{R}_i(s, a), & i\in \mathcal{L}(s),\\[2pt]
    0, & i\notin \mathcal{L}(s).
  \end{cases}  \nonumber
\end{equation}

Thus the POFSG $G$ can be embedded in a POSG $ G' = \langle \mathcal{I}, \mathcal{S}, \mathcal{A}^+, \mathcal{O}, \mathcal{T}^+, \mathcal{Z}, \mathcal{R}^+, \gamma \rangle $, inducing an $N$–player discounted stochastic game.
\end{lemma}

\section{Implementation Details}

\subsection{Predator--Prey}

\subsubsection{Network Architecture}
\paragraph{Shared convolutional encoder.}
All methods use a two-layer convolutional encoder with $3\times3$
kernels, stride $1$, and VALID padding, producing $8$ and $16$
channels, respectively. The spatial representation is flattened and
concatenated with non-spatial features before a linear layer of size
$128$ with ReLU activation.

\paragraph{Value-based methods (IQL, VDN).}
The encoder output is passed through a fully connected layer of size
$64$ with ReLU activation, followed by dueling value and advantage
heads. For VDN, joint action-values are obtained by summing
per-agent Q-values.

\paragraph{PPO.}
The encoder output is used as a shared embedding for actor and critic.
Both heads apply a linear layer of size $64$ with ReLU activation.
The actor outputs logits over actions and the critic outputs a scalar value.

\paragraph{MAPPO.}
The actor is identical to PPO. The critic is centralized and receives
stacked observations of all agents as input. It uses two convolutional
layers with $16$ output channels each, followed by fully connected
layers of sizes $128$ and $64$.

\paragraph{MAPPO\_state.}
The actor is identical to PPO. The centralized critic instead operates on a structured global state. In addition to the scalar quantities (number of active agents, prey alive, population cap) and the parent vector, it also receives the full grid representation. The grid is processed by two convolutional layers
with $16$ output channels, followed by fully connected layers of
sizes $128$ and $64$.

\subsubsection{Hyperparameter Optimization}

\begin{table*}[h]
\centering
\small
\begin{tabular}{lcc}
\toprule
 & \textbf{IQL / VDN} & \textbf{PPO / MAPPO / MAPPO\_state} \\
\midrule
Discount ($\gamma$) & 0.9, 0.99, 0.995 & 0.9, 0.99 \\
lr\_init & $1\mathrm{e}{-4}$, $5\mathrm{e}{-4}$, $1\mathrm{e}{-3}$ & $1\mathrm{e}{-4}$, $5\mathrm{e}{-4}$, $1\mathrm{e}{-3}$ \\
$\epsilon_{greedy}$ & 0.1, 0.2, 0.3 & -- \\
Max $\epsilon_{spawn}$ & 0.05, 0.1, 0.15, 0.2 & -- \\
\midrule
GAE $\lambda$ & -- & 0.9, 0.95 \\
Entropy coefficient & -- & 0.01, 0.05, 0.1 \\
Clip $\epsilon$ & -- & 0.05, 0.1, 0.15, 0.2 \\
\bottomrule
\end{tabular}
\caption{Hyperparameter search space for Predator--Prey.}
\end{table*}

\subsubsection{Hyperparameters Used}
\begin{table*}[h]
\centering
\small
\begin{tabular}{lccccc}
\toprule
\textbf{Hyperparameter} & \textbf{IQL} & \textbf{VDN} & \textbf{PPO} & \textbf{MAPPO} & \textbf{MAPPO\_state} \\
\midrule
lr\_init & $1\mathrm{e}{-3}$ & $1\mathrm{e}{-3}$ & $1\mathrm{e}{-4}$ & $1\mathrm{e}{-4}$ & $1\mathrm{e}{-4}$ \\
lr\_min & $1\mathrm{e}{-4}$ & $1\mathrm{e}{-4}$ & -- & -- & -- \\
Discount ($\gamma$) & 0.99 & 0.99 & 0.99 & 0.99 & 0.99 \\
Max grad norm & 1.0 & 1.0 & 0.5 & 0.5 & 0.5 \\
Rollout steps per epoch & 1 & 1 & 20 & 20 & 20 \\
Num environments & 1024 & 1024 & 1024 & 1024 & 1024 \\
Replay ratio & 16 & 16 & -- & -- & -- \\
\midrule
$\epsilon_{greedy}$ & 0.1 & 0.1 & -- & -- & -- \\
Max $\epsilon_{spawn}$ & 0.1 & 0.05 & -- & -- & -- \\
Spawn exploration schedule & linear & linear  & -- & -- & -- \\
\midrule

GAE $\lambda$ & -- & -- & 0.9 & 0.9 & 0.9 \\
Clip $\epsilon$ & -- & -- & 0.2 & 0.2 & 0.2 \\
Value loss coefficient & -- & -- & 0.5 & 0.5 & 0.5 \\
Entropy coefficient & -- & -- & 0.05 & 0.05 & 0.05 \\
Update epochs & -- & -- & 5 & 5 & 5 \\
Minibatches & -- & -- & 20 & 20 & 20 \\
\bottomrule
\end{tabular}
\caption{Selected hyperparameters for Predator--Prey.}
\end{table*}

\subsection{Level Based Foraging}

\subsubsection{Network Architecture}

\paragraph{Value-based methods (IQL, VDN).}
We use a dueling Q-network with LayerNorm and dropout.
The trunk consists of three linear layers with widths
$128 \rightarrow 256 \rightarrow 256$ with ReLU activations, each followed by LayerNorm.
A dropout layer (rate $0.1$) is applied in training.
Separate linear heads produce a scalar value and an advantage vector over actions.

\paragraph{PPO.}
We use an MLP actor--critic with a shared trunk of four linear layers
$64 \rightarrow 128 \rightarrow 256 \rightarrow 128$ with ReLU activations.
The actor is a linear head from the 128-d embedding to action logits (categorical policy),
and the critic is a linear head from the same embedding to a scalar value.
Actor and critic heads use orthogonal initialization (actor gain $0.01$, critic gain $1.0$). 

\paragraph{MAPPO.}
The actor matches PPO.
The critic is centralized: it receives the concatenated observations of all agents
(dimension $\texttt{obs\_dim}\times \texttt{num\_agents}$) and uses the same 4-layer MLP
$64 \rightarrow 128 \rightarrow 256 \rightarrow 128$ with ReLU activations, followed by
a scalar value head (orthogonal init, gain $1.0$) whose output is broadcast to all agents.

\subsubsection{Hyperparameter Optimization}
\begin{table*}[h]
\centering
\small
\begin{tabular}{lcc}
\toprule
 & \textbf{IQL / VDN} & \textbf{PPO / MAPPO} \\
\midrule
Discount ($\gamma$) & 0.9, 0.99, 0.995 & 0.9, 0.99 \\
lr\_init & $1\mathrm{e}{-4}$, $5\mathrm{e}{-4}$, $1\mathrm{e}{-3}$ & $1\mathrm{e}{-4}$, $5\mathrm{e}{-4}$, $1\mathrm{e}{-3}$ \\
$\epsilon_{greedy}$ & 0.1, 0.2, 0.3 & -- \\
Max $\epsilon_{spawn}$ & 0.05, 0.1, 0.15, 0.2 & -- \\
\midrule
GAE $\lambda$ & -- & 0.9, 0.95 \\
Entropy coefficient & -- & 0.01, 0.05, 0.1 \\
Clip $\epsilon$ & -- & 0.05, 0.1, 0.15, 0.2 \\
\bottomrule
\end{tabular}
\caption{Hyperparameter search space for LBF.}
\end{table*}

\subsubsection{Hyperparameters Used}

\begin{table*}[h]
\centering
\small
\begin{tabular}{lcccc}
\toprule
\textbf{Hyperparameter} & \textbf{IQL} & \textbf{VDN} & \textbf{PPO} & \textbf{MAPPO} \\
\midrule
lr\_init & $1\mathrm{e}{-3}$ & $5\mathrm{e}{-4}$ & $1\mathrm{e}{-3}$ & $1\mathrm{e}{-3}$ \\
lr\_min & $1\mathrm{e}{-4}$ & $1\mathrm{e}{-4}$ & -- & -- \\
Discount ($\gamma$) & 0.9 & 0.9 & 0.9 & 0.99 \\
Max grad norm & 1.0 & 1.0 & 0.5 & 0.5 \\
Rollout steps per epoch & 1 & 1 & 20 & 20 \\
Num environments & 10k & 10k & 20k & 20k \\
Batch size & $16384$ & $16384$ & -- & -- \\
\midrule

$\epsilon_{greedy}$ & 0.1 & 0.2  & -- & -- \\
Max $\epsilon_{spawn}$ & 0.1 & 0.1 & -- & -- \\
Spawn exploration schedule & linear & linear  & -- & -- \\
\midrule
GAE $\lambda$ & -- & -- & 0.95 & 0.9 \\
Clip $\epsilon$ & -- & -- & 0.2 & 0.2 \\
Value loss coefficient & -- & -- & 0.5 & 0.5 \\
Entropy coefficient & -- & -- & 0.01 & 0.01 \\
Update epochs & -- & -- & 5 & 5 \\
Minibatches & -- & -- & 5 & 5 \\
\bottomrule
\end{tabular}
\caption{Selected hyperparameters for LBF.}
\end{table*}

\subsection{PuddleBridge }

We use a convolutional dueling Q-network without parameter sharing.  The observation is split into a spatial base grid and a non-spatial remainder. The base grid is an $8\times 8$ map with $7$ channels and is processed by two
$3\times3$ convolutional layers with stride $1$ and VALID padding, producing
$8$ and $16$ channels, respectively, each followed by ReLU activation. The resulting representation is flattened. The non-spatial component consists of the agent’s position, agent ID, and a one-hot encoding of the previous joint action. These features are concatenated with the flattened convolutional representation. The fused representation is passed through two fully connected layers
of sizes $128$ and $64$ with ReLU activations, followed by dueling value and advantage heads.

\begin{table}[h]
\centering
\small
\begin{tabular}{ll}
\toprule
\textbf{Algorithm} & \textbf{VDN } \\
\midrule
lr\_init & $1\mathrm{e}{-3}$ \\
lr\_min & $1\mathrm{e}{-4}$ \\

Discount ($\gamma$) & 0.95 \\
Target update period  & 100 \\
Grad clip norm  & 1.0 \\

Num rollout steps per epoch  & 1 \\

$\epsilon_{greedy}$ & 0.3 \\
Max $\epsilon_{spawn}$ & 0.05 \\
Spawn exploration schedule & linear\\

Num environments  & 1024  \\
Replay ratio  & 16 \\
\bottomrule
\end{tabular}
\caption{Selected hyperparameters for PuddleBridge using VDN without parameter sharing.}
\end{table}

\section{Training Framework}

All experiments were implemented in JAX, with neural networks defined in Flax and optimization performed using the Adam optimizer from Optax. Experience replay for value-based methods was implemented using the Flashbax library.

\end{document}